%% file: bobw.tex
\documentclass[11pt,letterpaper]{article}

\usepackage[utf8]{inputenc}
\usepackage{mathrsfs}  
\usepackage{dsfont}
\usepackage{amssymb}
 \usepackage{amsmath}
 \usepackage{amsthm}
 \usepackage{bbm}
\usepackage{amsfonts}
\usepackage[left=1in,bottom=1in,top=1in,right=1in]{geometry}
\usepackage[colorlinks=true,citecolor=blue]{hyperref}
\usepackage{framed}
\usepackage[ruled]{algorithm2e}
\usepackage{soul}
\usepackage{thmtools,thm-restate}
\usepackage{csquotes}
\usepackage[nameinlink, noabbrev, capitalize]{cleveref}
\usepackage{booktabs}       
\usepackage{nicefrac}       
\usepackage{microtype} 
\usepackage{commath}
\usepackage{natbib}
\allowdisplaybreaks
\input{commands}

\author{%
Siddhartha Banerjee \\ORIE, Cornell\\\texttt{sbanerjee@cornell.edu} \\
 \and
 Alankrita Bhatt \\CMS, Caltech\\\texttt{abhatt@caltech.edu}\\
\and
 Christina Lee Yu\\ORIE, Cornell\\ \texttt{cleeyu@cornell.edu}
}

\newcommand{\Secref}[1]{\hyperref[#1]{Section \ref*{#1}}}
\newcommand{\Appref}[1]{\hyperref[#1]{Appendix \ref*{#1}}}
\crefname{equation}{}{}
\crefrangeformat{equation}{(#3#1#4) --~(#5#2#6)}
\crefmultiformat{equation}{(#2#1#3)}{, (#2#1#3)}{, (#2#1#3)}{, (#2#1#3)}
\crefname{lemma}{Lemma}{Lemmas}
\crefname{section}{Section}{Sections}
\crefname{subsubsubsection}{Section}{Sections}
\crefname{remark}{Remark}{Remarks}
\crefname{figure}{Figure}{Figures}
\crefname{table}{Table}{Tables}
\Crefname{lemma}{Lemma}{Lemmas}
\crefname{theorem}{Theorem}{Theorems}
\Crefname{theorem}{Theorem}{Theorems}

    \newtheorem{theorem}{Theorem}
    \newtheorem{corollary}{Corollary}
    \newtheorem{remark}{Remark}
    \newtheorem{example}{Example}
    \newtheorem{definition}{Definition}
    \newtheorem{proposition}{Proposition}
    \newtheorem{lemma}{Lemma}

\title{\textbf{The $\mathsf{SMART}$ Approach to Instance-Optimal Online Learning}}
\date{\vspace{-1em}}

\begin{document}

\maketitle
\begin{abstract}%
 \input{arxiv/arxiv_body/abstract_arxiv}
\end{abstract}
\input{arxiv/arxiv_body/introduction_arxiv}

\input{arxiv/arxiv_body/relatedwork_arxiv}
\input{arxiv/arxiv_body/twoapprox_arxiv}
\input{arxiv/arxiv_body/converse_arxiv}
\input{arxiv/arxiv_body/extensions_arxiv}
\input{arxiv/arxiv_body/conclusion_arxiv}

\section*{Acknowledgements}

This work is supported by NSF grants CNS-1955997, CCF-2337796 and ECCS-1847393, and AFOSR grant FA9550-23-1-0301. This work was partially done when the authors were visitors at the Simons Institute for the Theory of Computing, UC Berkeley. 

\appendix

\input{arxiv/arxiv_body/appendix_arxiv}

\bibliographystyle{unsrtnat}
 \bibliography{bobw}

\end{document}

%% file: commands.tex
\usepackage[capitalize]{cleveref}
\usepackage{tcolorbox}
\usepackage{dsfont}
\usepackage{enumitem}
\usepackage[bf]{caption}  
\usepackage{amssymb}

    \newcommand{\Prob}{\mathbb{P}}

    \newcommand{\Ber}{\mathrm{Bernoulli}}

    \newcommand{\defeq}{\mathrel{\mathop{:}}=}

    \newcommand{\g}{\gamma}
    \newcommand{\e}{\epsilon}
    \newcommand{\Reg}{\mathrm{Reg}}

    \newcommand{\wcalg}{\mathsf{ALG}_{\mathsf{WC}}}

    \newcommand{\zlast}{z_{\mathrm{last}}}

    \DeclareMathOperator\E{\mathbb{E}}
    \newcommand{\ftl}{\mathsf{FTL}}
    \newcommand{\cover}{\mathsf{Cover}}
    
    \newcommand{\Ac}{\mathcal{A}}
    \newcommand{\alg}{\mathsf{ALG}}

    \newcommand{\aftl}{a^{\mathsf{FTL}}}
    \newcommand{\ftlsum}{\Sigma^{\mathsf{FTL}}}
    \newcommand{\bob}{\mathsf{SMART}}
    
    \newcommand{\Majority}{\mathsf{Majority}}
    \newcommand{\Cover}{\mathsf{Cover}}
\allowdisplaybreaks

\newtheorem*{thm}{Theorem}

\usepackage{color-edits}
\addauthor{cy}{magenta}
\addauthor{ab}{cyan}
\addauthor{sb}{red}

%% file: arxiv/arxiv_body/abstract_arxiv.tex
We devise an online learning algorithm -- titled \emph{Switching via Monotone Adapted Regret Traces} ($\bob$) -- that adapts to the data and achieves regret that is \emph{instance optimal}, i.e., simultaneously competitive on every input sequence compared to the performance of the follow-the-leader ($\ftl$) policy and the worst case guarantee of any other input policy $\wcalg$.  
We show that the regret of the $\bob$ policy on \emph{any} input sequence is within a multiplicative factor $e/(e-1) \approx 1.58$ of the smaller of: 1) the regret obtained by $\ftl$ on the sequence, and 2) the upper bound on regret guaranteed by the given worst-case policy. This implies a strictly stronger guarantee than typical `best-of-both-worlds' bounds as the guarantee holds for every input sequence regardless of how it is generated. $\bob$ is simple to implement as it begins by playing $\ftl$ and switches at most once during the time horizon to $\wcalg$. Our approach and results follow from an operational reduction of instance optimal online learning to competitive anaylsis for the ski-rental problem. 
We complement our competitive ratio upper bounds with a fundamental lower bound showing that over all input sequences, no algorithm can get better than a $1.43$-fraction of the minimum regret achieved by $\ftl$ and the minimax-optimal policy.
We also present a modification of $\bob$ that combines $\ftl$ with a ``small-loss" algorithm to achieve instance optimality between the regret of $\ftl$ and the small loss regret bound. 

%% file: arxiv/arxiv_body/introduction_arxiv.tex
\section{Introduction} 
\label{sec:intro}

Our work aims to develop algorithms for online learning that are \emph{instance optimal}~\citep{fagin2001optimal},\citep[Chapter $3$]{roughgarden2021beyond} with respect to the stochastic and minimax optimal algorithms for a given setting. This is best motivated via
a concrete example:
\begin{example}[Binary Prediction]
We are given bit stream $y^n \defeq y_1,y_2,\ldots,y_n\in\{0,1\}^n$.
At the start of day $t$, before seeing $y_t$, we choose (possibly randomized) prediction $\widehat{Y}_t \sim \Ber(a_t)$ (for $a_t \in [0,1]$) for the upcoming bit $y_t$, given the history $y^{t-1}$. Our resulting loss on day $t$ is $\ell_t(a_t) = \Prob(\widehat{Y}_t \neq y_t) = |a_t - y_t|$, and our total loss is $L_n(\alg,y^n)\defeq \sum_{t=1}^n\ell_t(a_t)$. The objective is to achieve low \emph{regret} (i.e., additive loss) compared to the loss $L_n(a,y^n)=\sum_{t=1}^n\ell_t(a)$ of the best fixed action $a^* \in[0,1]$ in hindsight. As $a^*$ is the majority in $y^n$ between 0 and 1, it follows that $L_n(a^*, y^n) = \min\left\{\sum_{t=1}^n y_t, n - \sum_{t=1}^n y_t\right\}$. Formally, for sequence $y^n \in \{0,1\}^n$, policy $\alg$ incurs regret
\begin{align}
\label{eq:regdefn}
    \Reg(\alg, y^n) &\defeq L_n(\alg,y^n) - L_n(a^*,y^n)
    = \textstyle\sum_{t=1}^n |a_t - y_t| - \min_{a\in[0,1]}\textstyle\sum_{t=1}^T|a-y_t|.
\end{align}
\end{example}
Binary prediction goes back to the seminal works of~\cite{blackwell1956analog} and~\cite{hannan1957approximation}. The definition of regret is motivated by the case where $y_t$ is \emph{randomly} generated as i.i.d. Bernoulli$(p)$. If $p$ is known, then the optimal policy is the `Bayes predictor' $a^{\textsf{Bayes}}=\lfloor 2p\rfloor$ (i.e., nearest integer to $p$),
which coincides with hindsight optimal $a^*$ with high probability when $p$ is away from $1/2$. When $p$ is unknown, the \emph{stochastic optimal} policy is the \emph{Follow The Leader} or $\ftl$ policy, which sets $a_t=\Majority(y^{t-1})$, i.e. the majority bit amongst the first $t-1$ bits ($a_t=1/2$ if both are equal\footnote{We choose this specific tie-breaking rule for convenience; however, we can take any $a_t\in[0,1]$.}). 

A starting point for online learning is the observation that it is easy to construct a sequence $y^n$ such that $\ftl$ has poor regret: For example, if $y^n = (1,0,1,0,1,0,\ldots)$, i.e., alternate $1$s and $0$s, then the regret of $\ftl$ grows linearly with $n$. In contrast,~\emph{worst-case optimal} online learning policies such as those of Blackwell and Hannan, or more modern versions like Multiplicative Weights or Follow The Perturbed Leader (see~\cite{Cesa-Bianchi--Lugosi2006, slivkins2019introduction}) guarantee regret of $\Theta(\sqrt{n})$ over all sequences. Indeed, for bit prediction, the \emph{exact} minimax optimal policy was established by~\cite{Cover1966}, and this policy (which we refer to as $\cover$) achieves\footnote{Here $f_n$, the so-called Rademacher complexity of the setting, is a \emph{fixed} function of $n$ that does not depend on sequence $y^n$. For binary prediction, $f_n = \frac{\E|\sum_{t=1}^n Z_t|}{2} \approx \sqrt{\frac{n}{2\pi}}$ where $Z^n \sim \mathrm{Unif}\{1,-1\}$ i.i.d.}  $\Reg(\cover, y^n) = \sqrt{\frac{n}{2\pi}}(1+o(1))$ under \emph{any} $y^n \in\{0,1\}^n$, implying it is an \emph{equalizer} (achieves same regret over all sequences). 

While the above discussion seems a convincing endorsement of worst-case online learning algorithms, the situation is more complicated. One problem is that while $\ftl$ has bad regret on certain pathological sequences, on more `realistic' sequences $\ftl$ performs orders of magnitude better than the minimax regret. As an example, with i.i.d. Bernoulli$(p)$ input, $\Reg(\ftl, y^n)$ is actually \emph{independent of $n$} (i.e., $O(1)$) as long as $p$ is away from $1/2$ with high probability. We demonstrate this in~\cref{fig:SMART}$(a)$, where we see $\Reg(\ftl)$ is much lower than $\Reg(\cover)\approx 0.39\sqrt{n}$ unless $p$ is very close to $1/2$. This phenomena is known in more general settings~\citep{Huang--Lattimore--Gyorgy--Szapesvari2016}, suggesting that in practice one may be better off just using $\ftl$. On the other hand, as~\cref{fig:SMART}$(b,c)$ indicates, we know how to generate sequences $y^n$~\citep{feder1992universal} for which $\Reg(\ftl,y^n)$ grows linearly with $n$, and so the $\sqrt{n}$ regret of $\cover$ becomes appealing. 

Now suppose instead that a fictitious oracle is told beforehand which of $\ftl$ or $\cover$ is better suited for the upcoming sequence $y^n$; the demand made by \emph{instance optimality} is that we try to be competitive against such an oracle \emph{on every sequence $y^n$}.
\begin{definition}[Instance Optimality]\label{def:instanceopt_binary}
A binary prediction policy $\alg$ is instance optimal with respect to the regret of $\ftl$ and $\cover$ if there exists some universal $\gamma_n \geq 1$ such that for all $y^n \in \{0,1\}^n$:
\begin{equation*}
\label{eq:instanceopt_binary}
\small \Reg(\alg,y^n)\leq \gamma_n \min\{\Reg(\ftl,y^n),\Reg(\cover,y^n)\}
\end{equation*}
\end{definition}
We henceforth refer to $\gamma_n$ as the competitive ratio achieved by $\alg$; ideally we want this ratio to be a constant, i.e., $\gamma_n = O(1)$. 
This necessitates that on sequences where $\ftl$ gets a constant regret, then $\alg$ basically follows $\ftl$ throughout, while on sequences where $\ftl$ has high (in particular,~$\omega(\sqrt{n})$) regret, then $\alg$ follows $\cover$ in most rounds. 

The challenge in designing instance optimal algorithms is that the regret of any algorithm is a quantity that is not adapted to the natural filtration, i.e. it may not be possible to track $\Reg(\alg,y^n)$ for any $\alg$ from just the history $(y_1,y_2,\ldots,y_{t-1})$, since the hindsight optimal action $a^*$ depends on the \emph{entire} sequence $y^n$. 
One proxy is to track an algorithm's loss instead, leading to the idea of `corralling' policies~\citep{agarwal2017corralling,pacchiano2020model,Dann--Wei--Zimmert2023}, that run online learning over the reference algorithms to get within  $O(\text{poly}(n))$ of the smaller of the two losses.
Such an approach can not ensure $\gamma_n=O(1)$:
for example, consider an i.i.d. sequence of Bernoulli$(0.1)$ bits, where $\ftl$ has lower regret than $\cover$. With high probability on any such sequence we have small $\Reg(\ftl,y^n)=O(1)$ and yet high loss $L_n(a^*,y^n)=\Theta(n)$; now any corralling algorithm (even a small loss one) must suffer $O(\text{poly}(n))$ regret, and hence $\omega(1)$ competitive ratio. This example also shows that achieving a constant factor guarantee with respect to the minimum of the two \emph{losses} does not translate to a constant factor guarantee with respect to the minimum of the two \emph{regrets}. 

The instance optimal guarantee is closely related to \emph{best-of-both-worlds} guarantees, which aim for algorithms that simultaneously achieve (up to constant factors) both the low \emph{pseudoregret} guarantee of policies designed for stochastic inputs (as with $\ftl$ in our setting, or the Upper Confidence Bound ($\mathsf{UCB}$) algorithm in bandits), as well as a per-sequence regret guarantee comparable to a worst-case optimal algorithm $\wcalg$ (Eg. $\cover$ or Hedge in online learning; $\mathsf{EXP}3$ in bandits~\cite{auer2002nonstochastic}). 
Such guarantees have been shown in a variety of settings, including online learning~\citep{de2014follow,Orabona--Pal2015,mourtada2019optimality,bilodeau2023relaxing} and bandit settings~\citep{bubeck2012best,Zimmert--Seldin2019,Lykouris--Mirrokni--PaesLeme2018,Dann--Wei--Zimmert2023}.
One problem though is that since pseudoregret and worst-case regret are very different quantities, the above results tend to be hard to interpret, and less predictive of good performance\footnote{As an example, Hedge has optimal pseudoregret in certain stochastic settings~\citep{mourtada2019optimality}, but this is known to be sensitive to perturbations in the distributions~\citep{bilodeau2023relaxing}.}.
Note though that given a pair of stochastic/worst-case optimal algorithms, a policy that is $\gamma$-\emph{instance-optimal} w.r.t. these immediately satisfies a {best-of-both-worlds} guarantee with constant factor $\gamma$. 
In this regard, instance optimality provides a stronger guarantee as it holds on \emph{every} sequence $y^n$ regardless of how it is generated. 
Moreover, the parameter $\gamma$ can also provide sharper comparisons between algorithms, as well as admit hardness results on the limits of such guarantees.

\subsection{Our Contributions} \label{ssec:mainresults}

\begin{figure}[!t]
  \centering
  \begin{minipage}{.3\textwidth} 
    \centering
    \includegraphics[height=\linewidth]{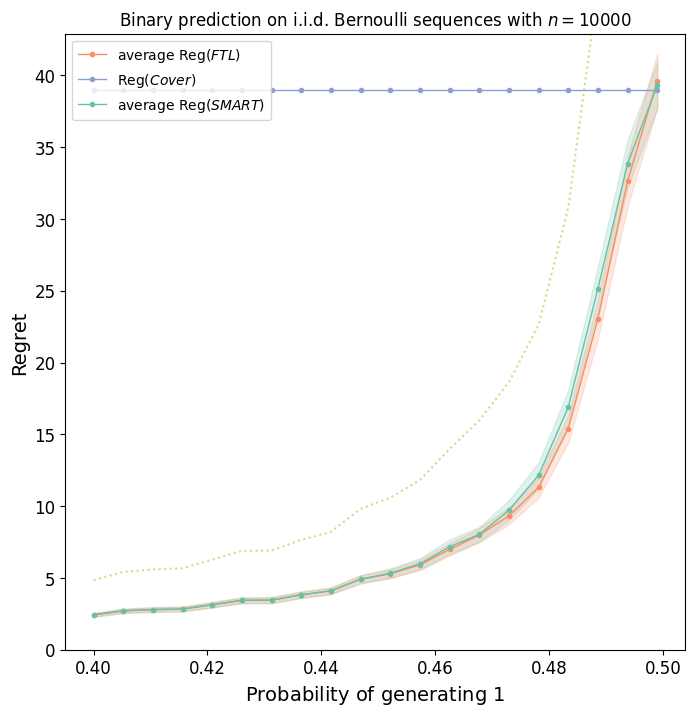}
  \end{minipage}%
  \hfill
  \begin{minipage}{.3\textwidth} 
    \centering
    \includegraphics[height=\linewidth]{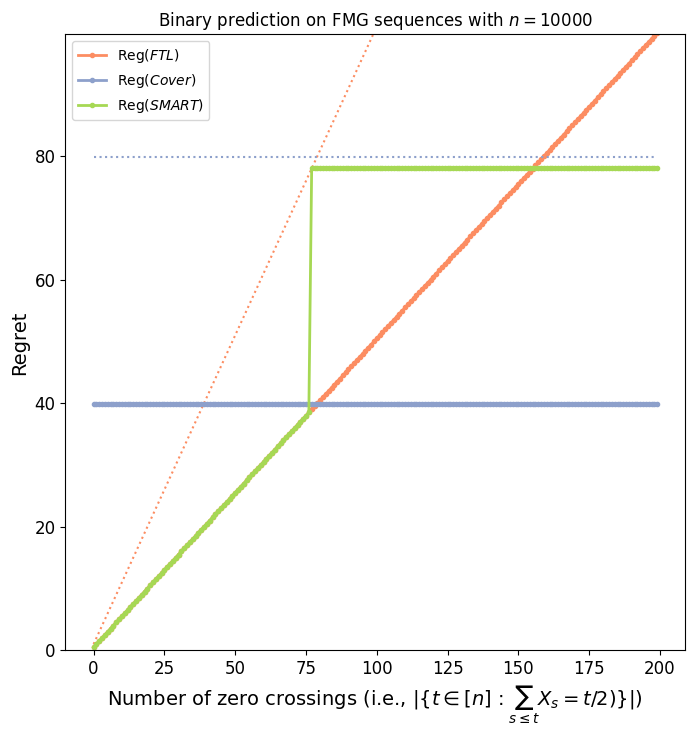}
  \end{minipage}%
  \hfill
  \begin{minipage}{.3\textwidth} 
    \centering
    \includegraphics[height=\linewidth]{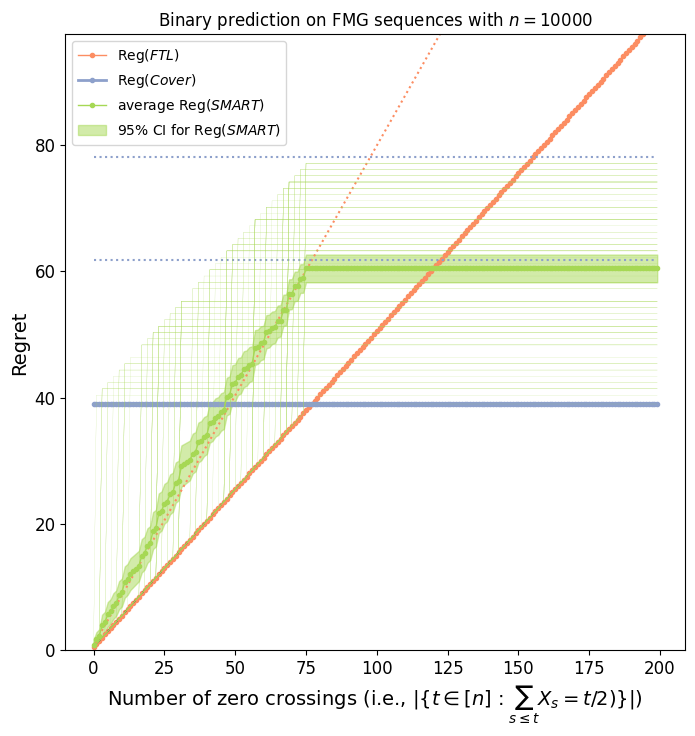}
  \end{minipage}
\captionsetup{format=plain}
  \caption{\it\footnotesize Comparing regret of $\ftl$, $\cover$ and $\bob$ on a collection of input sequences (for fixed $n$).\\ 
  $\bullet$ In Fig. $(a)$, we consider i.i.d. Bernoulli$(p)$ inputs for varying $p$. The regret of $\ftl$ is much lower than $\cover$ for $p<1/2$; the regret of $\bob$ tracks $\ftl$ closely (better than $2\Reg(\ftl)$, indicated by dotted line).\\
  $\bullet$ In Fig. $(b)$ and $(c)$, we consider `worst-case' binary sequences (as per~\citep{feder1992universal}) parameterized by the number of `lead-changes': the sequence with parameter $c$ comprises of $c$ pairs `$0,1$' or `$1,0$', followed by $n-2c$ `$1$'s. 
  In Fig. $(b)$, we consider $\bob$ with a deterministic switching threshold (\cref{thm:2ApproxRegret}) and compare $\Reg(\bob)$ with $2\Reg(\ftl)$ and $2\Reg(\cover)$ (dotted lines); in Fig. $(c)$, we use a randomized threshold (\cref{thm:SkiRentalRegret}), and show the average regret over the randomized threshold, as well as sample paths (plotted in green), and compare with $\frac{e}{e-1}$ times $\Reg(\ftl)$ and $\Reg(\cover)$ (dotted lines).}
  \label{fig:SMART}
  \vspace{-0.5cm}  
\end{figure}

We consider a general online learning setting where at the beginning of each round $t \in [n]$, a policy $\alg$ first plays an action $a_t \in \Ac$, following which, a loss function $\ell_t:\Ac\to [0,1]$ is revealed, resulting in a loss of $\ell_t(a_t)$. The \emph{regret} is defined according to:
\begin{align}
\label{eq:RegDefnGeneral}
    \Reg(\alg,\ell^n) = \textstyle\sum_{t=1}^n \ell_t(a_t) - \inf_{a \in \Ac} \textstyle\sum_{t=1}^n \ell_t(a).
\end{align} 
More generally, as in with bit prediction, we allow $\alg$ to play in round $t$ a measure $w_t\in\Delta_{\Ac}$ (i.e., play $\{w_t:\Ac\mapsto[0,1]|\sum_{a\in\Ac}w_t(a)=1\}$), resulting in an expected loss of $\sum_{a\in\Ac}w_t(a)\ell_t(a)$. For notational convenience, we henceforth use $(a_t,\ell_t(a_t))$ for the action/loss, and reserve use of expectations for randomness in the algorithm and/or sequence.

We want to understand when is it possible to attain instance optimality as in Eq.~\eqref{eq:instanceopt_binary} with respect to a given pair of algorithms. Ideally, we want the first to be optimal for stochastic instances, and the second to be minimax optimal; unfortunately however exact optimal policies are unknown except in simple settings. To this end, we make two amendments to our goal: First, for the stochastic optimal policy, we use $\ftl$; this is well defined in any online learning setting, and moreover, known to be optimal or near-optimal for a wide range of settings under minimal assumptions~\citep{kotlowski2018minimaxity}. Second, instead of the minimax policy, we use as reference \emph{any} policy $\wcalg$ which has a known worst case regret bound $g(n)$. With these modifications in place, we have the following objective.

\begin{definition} \label{def:instanceopt_wc}
Given $\ftl$ and any algorithm $\wcalg$ with $\sup_{\ell^n} \Reg(\wcalg, \ell^n) \le g(n)$, we say a policy $\alg$ is instance optimal with respect to the pair if there exists some universal $\gamma_n \geq 1$ (i.e. not depending on $y^n$) such that for every sequence of losses $\ell^n$:
\begin{equation*}\label{eq:instanceopt_wc}
\small \Reg(\alg,\ell^n)\leq \gamma_n \min\{\Reg(\ftl,\ell^n),g(n)\}
\end{equation*}
\end{definition}
While the above guarantee is not truly instance-optimal in that we are comparing against a worst-case regret bound $g(n)$ for $\wcalg$ rather than its performance on the instance $\ell^n$, the two are the same if $\wcalg$ is minimax optimal and hence attaining equal regret on all loss sequences; recall this is true of $\cover$  for binary prediction.

To realize the above goal, we propose the \emph{Switch via Monotone Adapted Regret Traces} ($\bob$) approach, which at a high level is a black-box way to convert design of instance-optimal policies into a simple \emph{optimal stopping} problem. 
Our approach depends on just two ingredients: first, owing to the additive structure of online learning problems, we have that the minimax guarantee $g(k)$ above holds over \emph{any} $k\in\mathbb{Z}$ and any (sub)sequence of $n$ loss functions; second, we show that $\ftl$ admits simple anytime regret estimator $\ftlsum_{\tau}$ (see Lemma~\ref{lem:FTLRegret}) which is \emph{monotone} and \emph{adapted} (i.e., a function only of historical data). Using these two observations, we can reduce the task of minimizing regret to a version of the `ski-rental' problem~\citep{karlin1994competitive,borodin2005online}, as follows:
we play $\ftl$ up to some stopping time $\tau$, and then switch to $\wcalg$ for the remaining $n-\tau$ periods, resetting all losses to zero. This algorithm  incurs a \emph{total} regret bounded by $\ftlsum_{\tau} + g(n-\tau)$, and using ideas from competitive analysis, we get that there is a simple way to choose the stopping time $\tau$ to achieve an $e/(e-1) \approx 1.58$-competitive ratio guarantee with respect to the minimum between the regret of $\ftl$ and the worst case guarantee $g(n)$.

\begin{thm}{\emph{\bf (See Theorem~\ref{thm:SkiRentalRegret})}}
      Let $\wcalg$ have worst-case regret $\sup_{\ell^n} \Reg(\wcalg, \ell^n) \le g(n)$
      where $g(n)$ is some monotonic function of $n$. An instantiation of $\bob$ achieves  
    \begin{align} \label{eq:2ApproxReg}
\small        \Reg(\bob, \ell^n) \le \frac{e}{e-1} \min\{\Reg(\ftl,\ell^n),g(n)\} + 1.
    \end{align}
\end{thm}
A highlight of our approach is the surprising simplicity of the algorithm and analysis, despite the strength of the instance optimality guarantee. In particular, our approach is modular, allowing one to plug in any $\wcalg$ and corresponding worst case bound $g(n)$, thus letting us handle any online learning setting with known minimax bounds. This results in an entire family of instance optimal policies for settings such as predictions with experts and online convex optimization. Moreover, the approach is easy to extend to get more complex guarantees; as an example, if $\wcalg$ is designed to get low regret for benign (i.e., `small-loss') sequences $\ell^n$, then we show how to use $\bob$ as a subroutine and achieve an instance optimal guarantee with respect to the regret of $\ftl$ and a small loss regret bound. 
\begin{corollary}[Following Theorem~\ref{thm:Regret_small_loss}]
\label{cor:SqrtLStarInstantiation}
    Consider the prediction with expert advice setting~\citep{cesa1997use}, where $\mathcal{A} = \Delta^{m-1}$, the $m-$simplex for $m \ge 2$, and $\ell_t(a) = \langle a, \ell_t \rangle $ for $\ell_t \in [0,1]^m$. Let $L^* \defeq \min_j \sum_{t=1}^n \ell_{tj}$. An instantiation of $\bob$ achieves
    \begin{align*}
        \Reg(\bob, \ell^n) 
        &\le 2 \min\left\{ \Reg(\ftl,\ell^n), 10 \sqrt{2L^* \log m} \right\} + O(\log L^* \log m).
    \end{align*}
\end{corollary}

Finally, studying instance optimality also lets us understand the fundamental limits of best-of-both worlds algorithms.
To this end, we provide a lower bound that shows our algorithm is nearly optimal in the competitive ratio. 
To the best of our knowledge, this is the first hardness result for best-of-both-worlds guarantees in online learning.
\begin{thm}{\emph{\bf (See Theorem~\ref{thm:LowerBdCR})}}
In the binary prediction setting, given \emph{any} online algorithm $\alg$, there exist sequences $y^n\in\{0,1\}^n$ such that:
\begin{align*}
    \Reg(\alg,y^n)
    &\geq 1.43\min\left\{\Reg(\ftl,y^n),\Reg(\cover,y^n)\right\}
\end{align*}
\end{thm}
Note again that in binary prediction, $\ftl$ achieves the optimal pseudoregret under i.i.d. inputs, while $\cover$ is the true minimax policy; thus, this is a fundamental lower bound on best-of-both-worlds guarantee in any online learning setting.

%% file: arxiv/arxiv_body/relatedwork_arxiv.tex
\subsection{Related work} \label{sec:relatedwork}
There have been many approaches towards combining stochastic and worst-case guarantees. As we discussed before, there is a large literature on best-of-both-worlds algorithms for both full and partial information settings~\citep{Wei--Luo2018, Bubeck--Li--Luo--Wei2019, Zimmert--Seldin2019, Dann--Wei--Zimmert2023}, and also more complex settings such as metrical task systems~\citep{bhuyan2023best} and control~\citep{sabag2021regret, goel2023best}. Another line of work~\citep{Rakhlin--Sridharan--Tewari2011, Haghtalab--Roughgarden--Shetty2022, Block--Dagan--Golowich--Rakhlin2022, Bhatt--Haghtalab--Shetty2023} considers \emph{smoothed analysis}, where the worst-case actions of the adversary are perturbed by nature. A third approach~\citep{bubeck2012best,Lykouris--Mirrokni--PaesLeme2018, Amir--Attias--Koren--Mansour--Livni2020, Zimmert--Seldin2019} interpolates between the stochastic and adversarial settings by considering most $\ell_t$ to be i.i.d., interspersed with a few adversarially chosen instances (corruptions). Finally, another line considers the data-generating distribution to come from a ball of specified radius around i.i.d. distributions~\citep{mourtada2019optimality, bilodeau2023relaxing}. While all these approaches provide useful insights into the gap between average and worst-case guarantees, one can argue they are all imprecisely specified -- given an instance $\{\ell_t\}_{t\in[n]}$ in hindsight, there is no clear sense as to which model best `explains' the instance.

Our focus on instance optimality instead follows the approach of better understanding and shaping the per-sequence regret landscape. The origins of this approach arguably come from the seminal work of~\cite{Cover1966} for binary prediction (we discuss this in more detail in~\cref{sec:converse}), with a later focus on better bounds for benign instances in general online learning~\citep{Auer--Cesa-Bianchi--Gentile2002,cesa2005minimizing,hazan2010extracting}. More recently, 
a line of work~\citep{Koolen--VanErven--Grunwald2014LLR, VanErven--Grunwald--Mehta--Reid--Williamson2015FastRates, VanErven--Koolen2016Metagrad, gaillard2014second} have studied designing policies that can adapt to different types of data sequences and achieve multiple performance guarantees simultaneously. The main idea is to use multiple learning rates that are weighted according to their empirical performance on the data. While the focus is still primarily on classifying instances based on when they are easier/harder to learn, some of the resulting guarantees have an instance-optimality flavor; for example,~\cite{VanErven--Koolen2016Metagrad} show how to simultaneously match the performance guarantee (in terms of certain variance bounds) attained by different learning rates in Hedge. Such approaches however need to understand their baseline algorithms in great detail, and use them in a `white-box' way to get their guarantees. 

In contrast, our approach fundamentally focuses on combining policies in a black-box way to get instance optimal outcomes. As we mention, this is similar in spirit to corralling bandit algorithms~\cite{agarwal2017corralling,pacchiano2020model,Dann--Wei--Zimmert2023}, as well as more recent work on online algorithms with predictions~\cite{bamas2020primal,dinitz2022algorithms,anand2022online}; however, as we mention above, these all get guarantees with respect to the loss of the reference algorithms, which is much weaker than our regret guarantees (though they do so in much more complex settings with partial information and/or state).
To our knowledge, the only previous result which attains a comparable instance-optimality guarantee to ours is that of~\cite{de2014follow} for the experts problem, where the authors propose the $\mathsf{FlipFlop}$ policy which interleaves $\mathsf{Hedge}$ (with varying learning rates) and $\mathsf{FTL}$ to obtain a regret guarantee similar to that of Corollary~\ref{cor:SqrtLStarInstantiation}. In fact, their guarantee is stronger as $\mathsf{FlipFlop}$ is shown to be 5.64-competitive with respect to $\min\{\Reg(\ftl,\ell^n), g(L^*)\}$ where $g(L^*) \le \sqrt{L^* \log m}$~\citep[Corollary 16]{de2014follow}. However, 
while $\mathsf{FlipFlop}$ depends on a clever choice of learning rates in $\mathsf{Hedge}$, $\bob$ can black-box interleave $\ftl$ with \emph{any} worst-case/small-loss algorithm without knowing the inner workings of said algorithm, which we see as a significant engineering strength. More importantly, our approach to this problem is distinct as we focus on the fundamental limits (upper and lower bounds) on the \emph{competitive ratio} that must be incurred when combining $\ftl$ with a worst-case policy; to the best of our knowledge this viewpoint, and the corresponding reduction to an optimal stopping problem, has not been previously explored. 

%% file: arxiv/arxiv_body/twoapprox_arxiv.tex
\section{Instance Optimal Online Learning: Achievability via $\bob$} \label{sec:twoapprox}

Given the setting and problem statement above, we can now present the $\bob$ policy. 
We do this for a general online learning problem, wherein we want to combine $\ftl$ with any given algorithm $\wcalg$ with a worst-case regret guarantee $
\sup_{\ell^n} \Reg(\wcalg, \ell^n) \le g(n)
$. 
Before presenting the policy, we first need the following regret decomposition for $\ftl$. 
\begin{lemma}[Regret of $\ftl$] \label{lem:FTLRegret}
If $L_t(\cdot) \defeq \sum_{i=1}^t \ell_t(\cdot)$, i.e. the cumulative loss function, then
\begin{align*} 
    \Reg(\ftl, \ell^n) = \textstyle\sum_{t=1}^n (L_t(a^*_{t-1}) - L_t(a^*_t)). 
\end{align*}
\end{lemma}
This is reminiscent of the `be-the-leader' lemma~\citep{kalai2005efficient,slivkins2019introduction}, although never stated explicitly as an exact decomposition. 
\begin{proof}
Recall we define $\aftl_t = a^*_{t-1}$, and hence $\inf_{a \in \Ac} \sum_{t=1}^n \ell_t(a)= L_n(a^*_n) $. Now we have 
    \begin{align*}
        \Reg(\ftl, \ell^n) 
        &= \textstyle\sum_{t=1}^n \ell_t(a^*_{t-1})  - L_n(a^*_n) \\
        &= \textstyle\sum_{t=1}^n (L_t(a^*_{t-1}) - L_{t-1}(a^*_{t-1})) -  L_n(a^*_n) \\
        &= \textstyle\sum_{t=1}^n (L_t(a^*_{t-1}) - L_t(a^*_t)) + \textstyle\sum_{t=1}^n (L_t(a^*_t) - L_{t-1}(a^*_{t-1})) - L_n(a^*_n) \\
        &\stackrel{(a)}{=} \textstyle\sum_{t=1}^n (L_t(a^*_{t-1}) - L_t(a^*_t))
    \end{align*}
    where $(a)$ follows since $\sum_{t=1}^n (L_t(a^*_t) - L_{t-1}(a^*_{t-1})) = L_n(a^*_n)$ by telescoping. 
\end{proof}
Next, let $\ftlsum_t$ denote the \emph{anytime} regret that $\ftl$ incurs up to round $t$ (i.e., assuming the game ends after round $t$). From Lemma~\ref{lem:FTLRegret} we have
\begin{align} 
\label{eq:FTLRegEquality}
 \ftlsum_t \defeq \Reg(\ftl, \ell^t) = \textstyle\sum_{i=1}^t (L_i(a^*_{i-1}) - L_i(a^*_i)) 
\end{align}
Now we make three critical observations: 
\begin{itemize}
    \item $\ftlsum_t$ is \textbf{adapted}: it can be computed at the end of the $t^{th}$ round 
    \item $\ftlsum_t$ is \textbf{monotone} non-decreasing in $t$ (since by definition $L_i(a^*_{i-1}) - L_i(a^*_i) \ge 0$)
    \item $\ftlsum_t$ is an \textbf{anytime lower bound} for $\Reg(\ftl,\ell^n)$, with $\ftlsum_n = \Reg(\ftl, \ell^n)$
\end{itemize}
For an input threshold $\theta \ge 0$ and an algorithm $\wcalg$, we get the following (meta)algorithm.

\begin{algorithm}[ht] 
\label{alg:bob}
\SetAlgoNoLine
\caption{Switching via Monotone Adapted Regret Traces ($\bob$)}
\label{alg:bob}
\KwIn{Policies $\ftl, \wcalg$, threshold $\theta$} 
\textbf{Initialize} $\ftlsum_0=0$, $t=1$ \;
\While{$\ftlsum_{t-1} \le \theta$}
{
    Set $a_t = a_{t-1}^{*}$%
         \tcp*{Play $\ftl$}
    Observe $\ell_t(\cdot)$\; 
    Update $L_t(\cdot)=L_{t-1}(\cdot)+\ell_t(\cdot)$ and $\ftlsum_t=\ftlsum_{t-1}+(L_t(a_{t-1}^{*})-L_t(a_{t}^{*}))$ and $t = t+1$\;
}
Reset losses to $0$ and play $\wcalg$ for remaining rounds (See~\cref{fig:proof_intuition}(b)) \;
\end{algorithm}

We now have the following performance guarantee for Algorithm~\ref{alg:bob} for $\theta=g(n)$.
\begin{theorem}[Regret of $\bob$ with deterministic threshold]  
\label{thm:2ApproxRegret}
      We are given $\ftl$, and any other policy $\wcalg$ with worst-case regret $\sup_{\ell^n} \Reg(\wcalg, \ell^n) \le g(n)$
      for some monotone function $g$. Then, playing $\bob$ with threshold $\theta = g(n)$ ensures 
    \begin{align} \label{eq:2ApproxReg}
        \Reg(\bob, \ell^n) \le 2 \min\{\Reg(\ftl,\ell^n),g(n)\} + 1.
    \end{align}
\end{theorem}

As we mention before, the structure of the $\bob$ algorithm (and the resulting guarantee) parallels the standard 2-competitive guarantee for the \emph{ski-rental} problem~\citep{karlin1994competitive}. 
This is a classical optimal stopping problem, where a principal faces an unknown horizon, and in each period must decide whether to rent a pair of skis for the period (for fixed cost \$1) or buy the skis for the remaining horizon (for fixed cost \$$B$). The aim is to design a policy which is minimax optimal (over the unknown horizon) with regards to the ratio of the cost paid by the principal, and the optimal cost in hindsight. 
We further expand on this connection in~\cref{ssec:binary} for the case of binary prediction. However, the connection suggests a natural follow-up question of whether randomized switching can help (as is the case for ski-rental); the following result answers this in the affirmative.

\begin{theorem}[Regret of $\bob$ with Randomized Thresholds]  
\label{thm:SkiRentalRegret}
We are given $\ftl$ and any other policy $\wcalg$ with worst-case regret $\sup_{\ell^n} \Reg(\wcalg, \ell^n) \le g(n)$
for some monotone function $g$. 
Moreover, given random sample $U\sim\text{Unif}[0,1]$, suppose we set
      \begin{align*}
      \theta=g(n)\ln(1+(e-1)U)
      \end{align*}
      Then playing the $\bob$ policy (Algorithm \ref{alg:bob}) with random threshold $\theta$ ensures 
    \begin{align*} 
        \E_{\theta}&\left[\Reg(\bob, \ell^n)\right]
        \le \frac{e}{e-1} \min\{\Reg(\ftl,\ell^n),g(n)\} + 1.
    \end{align*}
\end{theorem}
While we state the above as a randomized switching policy, this is more for interpretability -- it is easier to view our policy as switching between two black-box algorithms rather than playing a convex combination of the two. However, since we define the loss incurred by any $\alg$ in round $t$ to be the expected loss when $\alg$ plays a distribution $w_t$ over actions $\Ac$ (see~\cref{ssec:mainresults}), therefore we can alternately implement the above by mixing between the actions of $\ftl$ and $\wcalg$. More specifically, the above policy induces a \emph{monotone} mixing rule, where over $t$, the weight on the action suggested by $\wcalg$ is non-decreasing.

\begin{remark}[Optimality over Monotone Mixing Policies] 
\label{rem:SkiRentalLB}
The competitive ratio of $\frac{e}{e-1}$ is known to be optimal for the ski-rental problem via Yao's minmax theorem~\citep{borodin2005online}. A direct corollary of this is the optimality of $\bob$ over algorithms that are single switch (i.e., where the weight on $\wcalg$ is non-decreasing in $t$). One difference between our setting and ski-rental is that switching back-and-forth between $\ftl$ and $\wcalg$ is possible; see for example the $\mathsf{FlipFlop}$ policy of~\cite{de2014follow}. In~\cref{sec:converse} we provide a fundamental lower bound of $1.43$ on the competitive ratio over \emph{all} algorithms; this suggests that multiple switching can help get a better competitive ratio (since $e/(e-1)\approx 1.58$), but also, that single-switch policies are surprisingly close to optimal.
\end{remark}

\subsection{Illustrating the Reduction to Ski Rental in Binary Prediction}
\label{ssec:binary}

Before proving~\cref{thm:2ApproxRegret,thm:SkiRentalRegret}, we first illustrate the basic idea of our approach and reduction for the binary prediction setting. This is aided by the observation that the regret of $\ftl$ in this setting admits a simple geometric interpretation: for any sequence, and any time $t$, we have that $\ftlsum$ is equal to 1/2 times the number of `lead changes' up to time $t$; where a lead change is a time step $i$ where the count of $1$s and $0$s in the (sub)sequence $y^{i-1}$ is equal (see Figure~\ref{fig:proof_intuition}(a)). 
\begin{corollary}[$\ftl$ for binary prediction] \label{cor:BinPred}
In binary prediction, for any sequence $y\in\{0,1\}^n$ and any time $t\leq n$, define the lead-change counter
\[
c(y^t) \defeq |\{0 \le j \le t ~s.t.~ \sum_{i=1}^j y_i = \sum_{i=1}^j (1-y_i)\}|.
\]
Then we have $\ftlsum_t = \frac12 c(y^{t-1})$ and thus $\Reg(\ftl, y^n) = \frac12 c(y^{n-1})$. 
\end{corollary}

Corollary \ref{cor:BinPred} follows from the regret decomposition in Lemma \ref{lem:FTLRegret}. Furthermore, since the losses of 0s and 1s are equal at a lead change, it also follows that at a lead change $t$, the anytime regret $\ftlsum_t$ is also equal to the hindsight regret incurred by $\ftl$ up to time $t$. Since $\ftlsum_t$ only increases in value at lead changes, if the $\bob$ algorithm switches to $\wcalg$, it will only happen after a lead change, and thus $\bob$ behaves as if it had oracle knowledge of the regret of $\ftl$ from just the history up to the current time.

Consider the instantiation of $\bob$ where $\wcalg = \cover$ and $g(n) = \sqrt{\frac{n}{2 \pi}}$. As mentioned in Section~\ref{sec:intro}, a remarkable property of $\cover$ is that it is the true minimax optimal algorithm, where $\Reg(\Cover,y^n) = \sqrt{\frac{n}{2 \pi}} (1+o(1))$ for all sequences $y^n$, such that $g(n)$ is nearly equal to $\Reg(\Cover,y^n)$ regardless of the sequence~\citep{Cover1966}. 

It follows that $\bob$ is equivalent to an algorithm which starts with $\ftl$, plays it until the regret of $\ftl$ matches the minimax regret guaranteed by $\cover$ for the remaining sequence, and then switches to $\cover$ until the end. 
Let $t_{\mathrm{sw}}$ denote the last round $\bob$ plays $\ftl$ before switching to $\wcalg$ (with $t_{\mathrm{sw}}=n$ if it never switches).
For a single switch algorithm, the sequence which maximizes the regret is one that maximizes the $\ftl$ regret before the switch at $t_{\mathrm{sw}}$, and minimizes the $\ftl$ after the switch, as depicted in Figure \ref{fig:proof_intuition}(a). For such a sequence, $\ftlsum_t = t/4$ at lead changes $t$, such that the regret incurred by the algorithm is linear before the switch, matching the linear cost of renting skis in the ski rental problem. Note that $t_{\mathrm{sw}}$ will necessarily be $o(n)$ in such sequences as the time it takes until $\ftlsum_t \geq g(n)$ is linear in $g(n) = o(n)$. After the switch, $\cover$ will incur regret $\sqrt{\frac{n-t_{\mathrm{sw}}}{2 \pi}} (1+o(1)) = g(n) (1 + o(1))$, matching the fixed cost of buying skis at the switching point; Corollary \ref{cor:binary_pred_guarantees} follows as a result of this analysis. Furthermore, in the binary prediction setting, $\bob$ achieves the stronger notion of instance optimality stated in Definition \ref{def:instanceopt_binary}.

\begin{corollary} \label{cor:binary_pred_guarantees}
For all $y^n\in\{0,1\}^n$, $\bob$ with $\wcalg=\cover$ and $\theta = \sqrt{\frac{n}{2 \pi}}$ satisfies
\begin{align*}
    \Reg(\bob,y^n) &\leq 2\min\{\Reg(\ftl,y^n),\Reg(\cover,y^n)\} + 1.
\end{align*}
$\bob$ with $\wcalg=\cover$ and $\theta = \sqrt{\frac{n}{2\pi}}\ln(1+(e-1)U)$ for $U\sim\text{Uniform}[0,1]$ satisfies
\begin{align*}
    \Reg(\bob,y^n)
    &\leq 1.58\min\{\Reg(\ftl,y^n),\Reg(\cover,y^n)\} + 1.
\end{align*}
\end{corollary}

\begin{figure}[!t]
  \begin{minipage}{0.3\textwidth} 
    \centering
    \includegraphics[height=\linewidth]{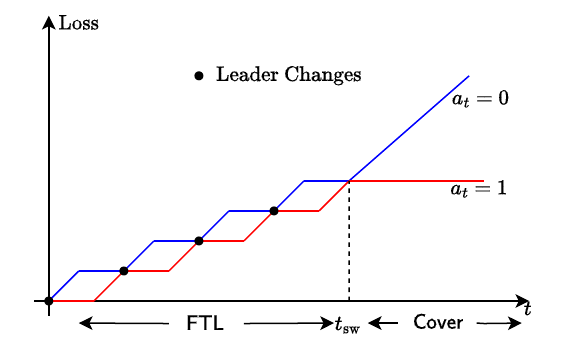}
  \end{minipage}%
  \hspace{1in}
  \begin{minipage}{.37\textwidth} 
    \centering
    \vspace{1.3em}
    \includegraphics[height=\linewidth]{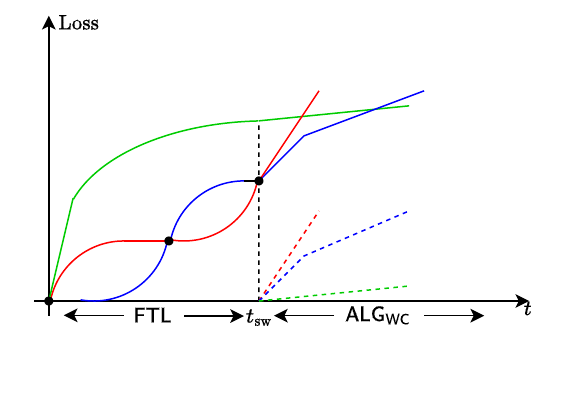}
  \end{minipage}%
\captionsetup{format=plain}
  \vspace{-0.5em}\caption{\it\footnotesize Figure \ref{fig:proof_intuition}(a) on the left shows the worst case instance in binary prediction for an algorithm which starts with $\ftl$ and switches at most once during the time horizon to $\cover$. Figure\ref{fig:proof_intuition}(b) on the right depicts in a prediction with experts setting how $\bob$ resets the losses after the switch from $\ftl$ to $\wcalg$.}
  \label{fig:proof_intuition}
  \vspace{-0.5cm}  
\end{figure}

\subsection{Proofs for General Online Learning}

In the general online learning setting, the proof is nearly the same, with the reduction to the ski-rental problem captured by Lemma \ref{lemma:regret_skirental}.
\begin{lemma} \label{lemma:regret_skirental}
Let $t_{\mathrm{sw}} \defeq \min_{1 \le t \le n-1} \ftlsum_t > \theta$ denote the last round $\bob$ plays $\ftl$ before switching to $\wcalg$ (with $t_{\mathrm{sw}}=n$ if it never switches). $\bob$ incurs regret bounded by
\[\Reg(\bob, \ell^n) \leq \Reg(\ftl, \ell^{t_{\mathrm{sw}}}) + \Reg(\wcalg, \ell_{t_{\mathrm{sw}}+1}^n) \leq \theta + \Reg(\wcalg, \ell_{t_{\mathrm{sw}}+1}^n) + 1.\]
\end{lemma}

\begin{proof}
We separately bound the regret of $\bob$ before the switch and after the switch,
\begin{align*}
\Reg(\bob, \ell^n)
&=  \big(\textstyle\sum_{t=1}^{t_{\mathrm{sw}}} \ell_t(a_t) - \textstyle\sum_{t=1}^{t_{\mathrm{sw}}} \ell_t(a^*_n)\big) + \big(\textstyle\sum_{t=t_{\mathrm{sw}}+1}^n \ell_t(a_t) - \textstyle\sum_{t=t_{\mathrm{sw}}+1}^n \ell_t(a^*_n)\big) \\
&\stackrel{(a)}{\le} \big(\textstyle\sum_{t=1}^{t_{\mathrm{sw}}} \ell_t(a_t) - \textstyle\sum_{t=1}^{t_{\mathrm{sw}}} \ell_t(a^*_{t_{\mathrm{sw}}})\big) + \big(\textstyle\sum_{t=t_{\mathrm{sw}}+1}^n \ell_t(a_t) - \textstyle\sum_{t=t_{\mathrm{sw}}+1}^n \ell_t(a^*_{t_{\mathrm{sw}}+1:n})\big) \\
&= \Reg(\ftl, \ell^{t_{\mathrm{sw}}}) + \Reg(\wcalg, \ell_{t_{\mathrm{sw}}+1}^n).
\end{align*}
The first term is upper bounded by $\Reg(\ftl, \ell^{t_{\mathrm{sw}}})$ since $L_{t_{\mathrm{sw}}}(a^*_n) \ge L_{t_{\mathrm{sw}}}(a^*_{t_{\mathrm{sw}}})$ by definition, as $a^*_{t_{\mathrm{sw}}}$ is the minimizer of $L_{t_{\mathrm{sw}}}$, and furthermore $\bob$ always plays according to $\ftl$ in rounds up to $t_{\mathrm{sw}}$. The second term is upper bounded by $\Reg(\wcalg, \ell_{t_{\mathrm{sw}}+1}^n)$ because $\sum_{t=t_{\mathrm{sw}}+1}^n \ell_t(a^*_{t_{\mathrm{sw}}+1:n}) \leq \sum_{t=t_{\mathrm{sw}}+1}^n \ell_t(a^*_n)$ as $a^*_{t_{\mathrm{sw}}+1:n}$ is the minimizer of the losses after $t_{\mathrm{sw}}$, and at time $t > t_{\mathrm{sw}}$, $\bob$ plays according to $\wcalg$ on the sequence of losses limited to $\ell_{t+1}^n$. This illustrates the important role of resetting the losses after the switch as depicted in Figure \ref{fig:proof_intuition}(b).
Using the fact that $\ftlsum_{t_{\mathrm{sw}}-1} \le \theta$ and $L_{t_{\mathrm{sw}}-1}(x^*_{t_{\mathrm{sw}}-1}) \le L_{t_{\mathrm{sw}}-1}(x^*_{t_{\mathrm{sw}}})$, it follows that \\
$~~~~~~\Reg(\ftl, \ell^{t_{\mathrm{sw}}}) = \ftlsum_{t_{\mathrm{sw}}-1} + L_{t_{\mathrm{sw}}-1}(a^*_{t_{\mathrm{sw}}-1}) - L_{t_{\mathrm{sw}}-1}(x^*_{t_{\mathrm{sw}}}) + \ell_{t_{\mathrm{sw}}}(a^*_{t_{\mathrm{sw}}-1}) - \ell_{t_{\mathrm{sw}}}(x^*_{t_{\mathrm{sw}}}) \leq \theta + 1.$
\end{proof}

The reduction to ski-rental is again immediate due to the properties that $\ftlsum_t := \Reg(\ftl,\ell^{t_{\mathrm{sw}}})$ is adapted, monotone, and is an anytime lower bound for $\Reg(\ftl,\ell^n)$, while remaining an upper bound on the true regret incurred by $\ftl$ up to time $t$. As a result, the algorithm can pretend that it truly observes the regret it incurs at each time up to the switching time $t_{\mathrm{sw}}$. After the switch, $\bob$ incurs regret $\Reg(\wcalg, \ell_{t_{\mathrm{sw}}+1}^n)$ which is upper bounded by $g(n-t_{\mathrm{sw}}) \leq g(n)$ by assumption.\\


\begin{proof}[Proof of Theorem~\ref{thm:2ApproxRegret}]
This follows immediately from Lemma \ref{lemma:regret_skirental}. For $\ell^n$ such that $\Reg(\ftl,\ell^n) \le g(n)$, $\bob$ will never switch to $\cover$ as $\ftlsum_t \le \Reg(\ftl, \ell^n) \leq g(n)$ such that $\Reg(\bob,\ell^n) = \min\{\Reg(\ftl, \ell^n), g(n)\}$. For $\ell^n$ such that $\Reg(\ftl,\ell^n) > g(n)$, by Lemma \ref{lemma:regret_skirental}, $\Reg(\bob, \ell^n)\leq g(n) + g(n-t_{\mathrm{sw}}) +1 \leq 2 g(n) +1$. \\
\end{proof}

%

\begin{proof}[Proof of Theorem~\ref{thm:SkiRentalRegret}]
The proof uses a primal-dual approach, similar to that of~\cite{karlin1994competitive} for ski-rental.
For a given sequence of loss functions $\ell^n$, we use the shorthand $r= \Reg(\ftl,\ell^n)$ and $g=g(n)$.
Also, for our given choice of cumulative distribution function $F_n$, the corresponding probability density function is given by $f(z) = \frac{e^{z/g}}{g(e-1)}$ for $z\in[0,g]$. 
As before, let $t_{\mathrm{sw}} \defeq \min_{1 \le t \le n-1} \ftlsum_t > \theta$ be the (random) round where $\bob$ switches from $\ftl$ to $\wcalg$ (with $t_{\mathrm{sw}}=n$ if it never switches). Then by Lemma \ref{lemma:regret_skirental} we have
\begin{align*}
\frac{\Reg(\bob,\ell^n) -1}{\min\{\Reg(\ftl,\ell^n),g(n)\}}\leq \begin{cases} 
   \frac{\theta + g}{\min\{r,g\}} & \text{if } t_{\mathrm{sw}} < n \\
   1       & \text{if } t_{\mathrm{sw}} = n
  \end{cases}
\end{align*}
where the second case follows from the fact that $\theta\in[0,g]$, and hence if we never switch, then $r\le g(n)$. 
Taking expectation over $\theta$, we have
\begin{align*}
&\frac{\E_{\theta}\left[\Reg(\bob,\ell^n)\right]-1}{\min\{\Reg(\ftl,\ell^n),g\}}
\le 
\begin{cases} 
   \int_0^r   \frac{x + g}{r}f(x)dx + 1 - F(r)  & \text{if } r \le g \\
   \int_0^g   \frac{x + g}{g}f(x)dx        & \text{if } r > g
  \end{cases}
\end{align*}
Let $\phi(z) \defeq  \int_0^z \frac{(x + g)}{z}f(x)dx + 1 - F(z) $ for $z\in[0,g]$; then $\frac{\E_{\theta}\left[\Reg(\bob,\ell^n)\right]-1}{\min\{\Reg(\ftl,\ell^n),g\}} \leq \max_{z\in[0,g]}\phi(z)$. Moreover, we can differentiate to get $z^2\phi'(z)= gzf(z) - \int_0^z(x + g)f(x)dx$. Substituting our choice of $f$ in this expression, we get
\begin{align*}
\frac{z^2d\phi(z)}{dz} &= \frac{ze^{z/g}}{(e-1)} - \int_0^z(x + g)\frac{e^{x/g}}{g(e-1)}dx
=\frac{ze^{z/g} - g\int_0^{z/g}(w + 1)e^{w}dw}{(e-1)}=0
\end{align*}
Thus, $\phi(z)$ is constant for all $z\in[0,g]$ and $\phi(g) = \frac{1}{g(e-1)}\int_0^g(1+x/g)e^{x/g}dx = \frac{e}{e-1}$. 
\end{proof}

%% file: arxiv/arxiv_body/converse_arxiv.tex
\section{Instance Optimal Online Learning: Converse} \label{sec:converse} 
In this section, we investigate fundamental limits on the instance-optimal regret guarantees achievable by \emph{any} algorithm. More precisely, in the setting of binary prediction, we ask what is the smallest value of $\gamma_n$ satisfying
\begin{align} \label{eq:LBQues}
    \Reg(\alg, y^n) &\le \g_n \min\{\Reg(\ftl,y^n),\Reg(\Cover,y^n)\}
    = \g_n \min\left\{\tfrac12 c(y^{n-1}), f_n\right\}
\end{align}
for all $y^n$, where $f_n \defeq  \Reg(\Cover,y^n) = \sqrt{\frac{n}{2 \pi}} (1+o(1))$. 
We show the following lower bound.
\begin{theorem}[Lower bound on the competitive ratio] \label{thm:LowerBdCR}
    \[
    \lim_{n \to \infty} \g_n \ge \Big(1-e^{-1/\pi}+2Q\big(\sqrt{2/\pi}\big)\Big)^{-1} \approx 1.4335
    \]
    where the $Q(\cdot)$ function is $Q(x) \defeq \frac{1}{\sqrt{2\pi}}\int_{x}^{\infty} e^{-t^2/2}$. 
\end{theorem}
Since binary prediction is a specific online learning problem, this also yields a fundamental lower bound for instance-optimality for general online learning.
Note particularly that $\g_n > 1$, implying that $(1+o(1))\min\{\Reg(\ftl),\Reg(\wcalg)\}$ regret is not possible to achieve. Thus, there is an inevitable multiplicative factor that must be paid in order to achieve an instance-optimal regret guarantee.

An equivalent way to state~\cref{eq:LBQues} is to find the smallest $\g_n$ for which a predictor $\{a_t(y^{t-1})\}_{t=1}^n$ satisfies for all $y\in\{0,1\}^n$
\begin{align}\label{eq:InstanceOptLossFn}
   \textstyle\sum_{t=1}^n |a_t(y^{t-1})-y_t| 
   &\le \g_n \min\big\{\tfrac12 c(y^{n-1}), f_n\big\} + \min\big\{\textstyle\sum_{t=1}^n y_i, n-\textstyle\sum_{t=1}^n y_i\big\}. 
\end{align}
In order to establish the values of $\gamma_n$ for which the loss function in the right hand side of~\cref{eq:InstanceOptLossFn} are achievable, we utilize the following result of~\cite{Cover1966}, which provides an exact characterization of the set of \emph{all} loss functions achievable in binary prediction. Formally, we say a function $\phi : \{0,1\}^n \to \mathbb{R}^+$ is \emph{achievable} in binary prediction if there exists a predictor/strategy $a_t: y^{t-1} \mapsto [0,1]$ that ensures $\sum_{t=1}^n |a_t(y^{t-1})-y_t| = \phi(y^n)\quad,\,\forall\,y^n\in\{0,1\}^n$. 
Then, we have the following characterization.
\begin{theorem}[\cite{Cover1966}]  
\label{thm:CoverAchievableFn}
Let $\e^n \sim \mathrm{Bern}\left(\frac{1}{2}\right)$ i.i.d. For $\phi$ to be achievable, it must satisfy the following:
\begin{itemize}
    \item Balance:  $\E[\phi(\e^n)] = \frac{n}{2}$ . 
    \item Stability: Let $\phi_t(y^t) \defeq \E[\phi(y^{t}\e_{t+1}^n)]$; then $|\phi_t(y^{t-1}0)-\phi_t(y^{t-1}1)| \le 1\,\forall t\in [n], \,y^t\in\{0,1\}^t$. 
    \end{itemize}
    Further any $\phi$ satisfying the above is realized by predictor  $a_t(y^{t-1}) = \frac{1+\phi_t(y^{t-1}0)-\phi_t(y^{t-1}1)}{2}
    $.
\end{theorem}
As an immediate corollary, Theorem~\ref{thm:CoverAchievableFn} equips us with the \emph{exact} minimax optimal algorithm for binary prediction alluded to in~\cref{sec:intro}. 
Returning to our setting, from the balance condition in~\cref{thm:CoverAchievableFn}, for $\epsilon^n \sim \mathrm{Bern}(1/2)$ i.i.d. 
\begin{align*}
\g_n \E\big[\min\big\{\tfrac12 c(\e^{n-1}), f_n\big\}\big] 
+ \E\big[ \min\big\{\textstyle\sum_{t=1}^n \e_t, n-\textstyle\sum_{t=1}^n \e_t\big\}\big] \ge \frac{n}{2}.
\end{align*}
for the function in~\eqref{eq:InstanceOptLossFn} to be achievable. Using the definition of $f_n$, 
\begin{align*}
    \g_n \ge \frac{f_n}{\E\left[ \min\left\{\tfrac12 c(\e^{n-1}), f_n\right\}\right]} =  \frac{2f_n}{\E\left[\min\left\{c(\e^{n-1}), 2f_n\right\}]\right]}.
\end{align*}
The above bound immediately yields that $\g_n \ge 1$ as expected. We can further sharply characterize the asymptotics of $\g_n$, resulting in the stated lower bound. The full proof of Theorem~\ref{thm:LowerBdCR} is provided in Appendix~\ref{appendix:ConversePf}.

%% file: arxiv/arxiv_body/extensions_arxiv.tex
\section{Instance-Optimal Algorithms in Small-Loss Settings} \label{sec:SmallLoss}


So far, we have presented specializations of $\bob$ that achieve instance-optimality between $\ftl$ and the \emph{worst-case} regret $g(n)$. However, many worst-case algorithms can still adapt to the instance $\ell^n$ and achieve regret guarantees that are a function of the `difficulty' of the instance $\ell^n$. A common way to quantify this is difficulty is via \emph{small-loss bounds}, where the regret is upper bounded by $g(L^*)$ where $g(\cdot)$ as earlier is a monotonic increasing function and $L^* \defeq \min_{a \in \Ac} \sum_{t=1}^n \ell_{t}(a)$ is the loss achieved by the best action. Such guarantees imply that for sequences where there exists an action achieving low loss, the corresponding regret achieved is also low. Thus, a natural question is whether $\bob$ can be specialized to yield an algorithm that is constant competitive with respect to $\min\{\Reg(\ftl, \ell^n), g(L^*)\}$.  

As a starting point, if $L^*$ is known apriori, it is easy to achieve a $\frac{e}{e-1}$ approximation by simply using $\bob$ with (random) threshold $\theta = g(L^*)\ln(1+(e-1)U), U \sim \mathrm{Unif}[0,1]$; this is an immediate corollary of Theorem~\ref{thm:SkiRentalRegret}. When $L^*$ is not known, we use a guess-and-double argument to devise an algorithm that achieves the following instance-optimality guarantee. 

\begin{theorem}[Regret of $\bob$ for unknown small loss]  
\label{thm:Regret_small_loss}
      Let $\wcalg$ have small loss regret guarantees satisfying $\Reg(\wcalg, \ell^n) \le g(L^*)$ 
    for any $\ell^n$ where $L^* = \min_{j \in [m]} L_{tj}$, i.e. the loss achieved by the best expert in hindsight. Then, if we play $\bob$ for Small-Loss as stated in~Algorithm $2$, we have 
    \begin{align*}
    \Reg(\bob, \ell^n) 
    &\leq 2 \min\big(\Reg(\ftl, \ell^n), \textstyle\sum_{z=1}^{\log(1+L^*/\log m) +1} g(2^z \log m)\big) + O(\log L^*/\log m)
    \end{align*}
\end{theorem}
In particular, in the prediction with expert advice setting, we know that $\mathsf{Hedge}$ with a time-varying learning rate achieves $g(L^*) \equiv 2\sqrt{2 L^* \log m} + \kappa \log m$ (where $\kappa > 0$ is an absolute constant)~\citep{Auer--Cesa-Bianchi--Gentile2002, cesa2007improved}; this gives Corollary~\ref{cor:SqrtLStarInstantiation} in Section~\ref{sec:intro}. 

The intuitive idea behind the algorithm is to guess the value of $L^*$, and play $\bob$ with this guessed value while simultaneously keeping track of the regret incurred. Whenever the regret incurred exceeds the guarantee established by $\bob$ with known $L^*$ double the guessed value and start again. We use the notation $\wcalg(\ell_{t_1}^{t_2})$ to refer to the worst-case algorithm when the previously observed sequence is $\ell_{t_1}^{t_2}$; in particular this would be equivalent to the action recommended at time $t_2 + 1$ after throwing away all the observed losses before $t_1$. 
We let $\ftlsum_{t_1:t_2} = 
\sum_{i=t_1}^{t_2} (L_i(a^*_{i-1}) - L_i(a^*_i))$, which grows as the number of leader changes within $i \in [t_1, t_2]$. The algorithm's pseudocode is given in Algorithm 2 below, and a proof of Theorem~\ref{thm:Regret_small_loss} is provided in Appendix~\ref{appendix:smallLossPf}.

\begin{algorithm}[!ht]
\SetAlgoNoLine
\caption{$\bob$ for Small-Loss}
\label{alg:smallLossSMART}
\KwIn{Policies $\ftl, \wcalg$; Small-loss bound $g(\cdot)$} 
\For{$z = 0, 1, \dotsc$ (epochs)}
{
    Let $t = t_z\defeq$start time of $z^{th}$ epoch, $L^*_z \defeq 2^z \log m$ (current guess for $L^*$), $\ftlsum_{t_z:t_z-1}=0$ \;

    \While{$\ftlsum_{t_z:t-1} \le g(L^*_z)$}
    {
        Set $a_t = a_{t-1}^{*}$%
             \tcp*{Play $\ftl$}
        Observe $\ell_t(\cdot)$\; 
        Update $L_t(\cdot)=L_{t-1}(\cdot)+\ell_t(\cdot)$ and $\ftlsum_{t_z:t}=\ftlsum_{t_z:t-1}+(L_t(a_{t-1}^{*})-L_t(a_{t}^{*}))$ and $t = t+1$\;
    }

    Let $\tau_z := \min_{t \geq t_z} \ftlsum_{t_z:t} > g(L^*_z)$ and $t = \tau_z + 1$\;
    \tcp{Check if loss incurred by $\wcalg$ in this epoch violates the upper bound from $L_z^*$ is correct}
    \While{$\sum_{t=t_z}^t \langle a_t, \ell_t \rangle \leq L_z^* + 2\min\{\ftlsum_{t_z:t}, g(L_z^*)\} + 1$}
    {
        Set $a_t = \wcalg(\ell_{\tau_z+1}^{t-1})$ \tcp*{Play $\wcalg$ forgetting losses before $\tau_z+1$}
        Observe $\ell_t(\cdot)$\; 
        Update $L_t(\cdot)=L_{t-1}(\cdot)+\ell_t(\cdot)$ and $\ftlsum_{t_z:t}=\ftlsum_{t_z:t-1}+(L_t(a_{t-1}^{*})-L_t(a_{t}^{*}))$ and $t = t+1$\;
    }    
}
\end{algorithm}

%% file: arxiv/arxiv_body/conclusion_arxiv.tex
\section{Conclusion} \label{sec:conclusion}

In this paper, we present $\bob$, a simple and black-box online learning algorithm that adapts to the data and achieves instance optimal regret with respect to $\ftl$ and any given worst-case algorithm. We show that $\bob$ only switches once from $\ftl$ to the worst-case algorithm, and attains a regret that is within a factor of $e/(e-1) \approx 1.58$ of the minimum of the regret of FTL and the minimax regret over all input sequences; we also show that any algorithm must incur an extra factor of at least $1.43$ establishing that our simple approach is surprisingly close to optimal. Furthermore, we extend SMART to incorporate a small-loss algorithm and obtain instance optimality with respect to the small-loss regret bound. Our approach relies on a novel reduction of instance optimal online learning to the ski-rental problem, and leverages tools from information theory and competitive analysis. Our work suggests several open problems for future research, such as finding instance optimal algorithms for bandit settings, or designing algorithms that can adapt to multiple reference algorithms besides FTL and minimax algorithms.

%% file: arxiv/arxiv_body/appendix_arxiv.tex
\section{Omitted proofs from Section~\ref{sec:SmallLoss}} \label{appendix:smallLossPf}
In this Section, we will establish the proofs of Theorem~\ref{thm:Regret_small_loss} and Corollary~\ref{cor:SqrtLStarInstantiation}. 

Recall Algorithm 2, where $\wcalg(\ell_{t_1}^{t_2})$ refers to the worst-case algorithm when the previously observed sequence is $\ell_{t_1}^{t_2}$; in particular this would be equivalent to the action recommended at time $t_2 + 1$ after throwing away all the observed losses before $t_1$. 
We let $\ftlsum_{t_1:t_2} = 
\sum_{i=t_1}^{t_2} (L_i(a^*_{i-1}) - L_i(a^*_i))$, which grows as the number of leader changes within $i \in [t_1, t_2]$.

We first have the following decomposition of the regret for any algorithm $\alg$ that plays the sequence of actions $a^{\alg}_t$ at time $t$.
\begin{lemma} \label{regret_decomp}
The regret incurred by any sequence of actions $(a^\alg_t)_{t\in[n+1]}$ can be written as
\begin{align}
\Reg(\alg, \ell^n)
&= \sum_{t=1}^{n} \left(L_t(a^{\alg}_t) - L_{t}(a^{\alg}_{t+1}) \right),
\end{align}
where we let $a_{n+1}^{\alg} := a_n^*$.
\end{lemma}

\begin{proof}
\begin{align}
L_n(\alg)
&= \sum_{t=1}^n \ell_t(a^{\alg}_t) \\
&= \sum_{t=1}^n \left(L_t(a^{\alg}_t) - L_{t-1}(a^{\alg}_t) \right) \\
&= L_n(a^{\alg}_n) + \sum_{t=1}^{n-1} \left(L_t(a^{\alg}_t) - L_{t}(a^{\alg}_{t+1}) \right) - L_0(a^{\alg}_1).
\end{align}
This implies a regret decomposition of
\begin{align}
\Reg(\alg, \ell^n)
&= L_n(\alg) - L_n(a_n^*) \\
&= L_n(a^{\alg}_n) - L_n(a_n^*) + \sum_{t=1}^{n-1} \left(L_t(a^{\alg}_t) - L_{t}(a^{\alg}_{t+1}) \right) - L_0(a^{\alg}_1)
\end{align}
As $L_0(a^{\alg}_1) = 0$, and $a_{n+1}^{\alg} := a_n^*$, it follows that 
\begin{align}
\Reg(\alg, \ell^n)
&= \sum_{t=1}^{n} \left(L_t(a^{\alg}_t) - L_{t}(a^{\alg}_{t+1}) \right).
\end{align}
\end{proof}

Next, we use this decomposition to establish the regret of any algorithm $\alg$ that alternates between playing $\ftl$ and another algorithm $\wcalg$.

\begin{lemma} \label{small_loss_regret_lemma}
Consider an algorithm $\alg$ which alternates between playing $\ftl$ and $\wcalg$, where $\ftl$ is played in the intervals $\{[t_z, \tau_z]\}_{z\in[\zlast]}$, and $\wcalg$ is played in intervals $\{[\tau_z +1, t_{z+1}-1]\}_{z\in[\zlast]}$. The regret of $\alg$ is bounded by
\begin{align}
\Reg(\alg, \ell^n) \leq
\sum_z \left(\ftlsum_{t_z:\tau_z-1} + \Reg(\wcalg, \ell_{\tau_z+1}^{t_{z+1}-1}) + 1\right).
\end{align}
\end{lemma}

\begin{proof}
We let $t_{\zlast} = n+1$, and $a_{n+1} = a_n^*$. We use~\cref{regret_decomp} and rearrange the terms by grouping them by the FTL periods and the $\wcalg$ periods.
\begin{align}
\Reg&(\alg, \ell^n) \nonumber\\
&= \sum_z \left( \sum_{t=t_z}^{t_{z+1}-1} (L_t(a_t) - L_t(a_{t+1})) \right) \nonumber\\
&= \sum_z \left( \sum_{t=t_z}^{\tau_z-1} (L_t(a_{t-1}^*) - L_t(a_{t}^*))
+ L_{\tau_z}(a_{\tau_z-1}^*) - L_{\tau_z}(a_{\tau_z+1}) + \sum_{t=\tau_z+1}^{t_{z+1}-1} (L_t(a_t) - L_t(a_{t+1})) \right) \nonumber\\
&= \sum_z \left( \sum_{t=t_z}^{\tau_z-1} (L_t(a_{t-1}^*) - L_t(a_t^*)) + L_{\tau_z}(a_{\tau_z-1}^*) + \sum_{t=\tau_z+1}^{t_{z+1}-1} \ell_t(a_t) - L_{t_{z+1}-1}(a_{t_{z+1}}) \right) \nonumber\\
&= \sum_z \sum_{t=t_z}^{\tau_z-1} (L_t(a_{t-1}^*) - L_t(a_{t}^*)) + \sum_z \left(\sum_{t=\tau_z+1}^{t_{z+1}-1} \ell_t(a_t) + L_{\tau_z}(a_{\tau_{z-1}}^*) - L_{t_{z+1}-1}(a_{t_{z+1}-1}^*) \right). \nonumber
\end{align}

We bound the first term by the FTL regret. Recall the notation
\begin{align}
\ftlsum_{t_z:\tau_z-1} := \sum_{t=t_z}^{\tau_z-1} (L_t(a_{t-1}^*) - L_t(a_{t}^*)).
\end{align}
Because we are playing FTL at both time $\tau_z$ and $\tau_z - 1$, it holds that 
\[L_{\tau_z}(a_{\tau_{z}-1}^*) = L_{\tau_z-1}(a_{\tau_{z}-1}^*) + \ell_{\tau_z}(a_{\tau_{z}-1}^*) \leq L_{\tau_z}(a_{\tau_{z}}^*) + 1.\]

To bound the second term, we will show that 
\begin{align}
\sum_{t=\tau_z+1}^{t_{z+1}-1} &\ell_t(a_t) + L_{\tau_z}(a_{\tau_{z-1}}^*) - L_{t_{z+1}-1}(a_{t_{z+1}-1}^*) 
\nonumber\\
&\leq \sum_{t=\tau_z+1}^{t_{z+1}-1} \ell_t(a_t) + L_{\tau_z}(a_{\tau_{z}}^*) - L_{t_{z+1}-1}(a_{t_{z+1}-1}^*) + 1 \\ 
&= \sum_{t=\tau_z+1}^{t_{z+1}-1} \ell_t(a_t) + \min_a L_{\tau_z}(a)  - \min_a L_{t_{z+1}-1}(a) +1\\ 
&\leq \sum_{t=\tau_z+1}^{t_{z+1}-1} \ell_t(a_t) - \min_a \left(L_{t_{z+1}-1}(a) - L_{\tau_z}(a)\right) + 1 \\
&= \Reg(\wcalg, \ell_{\tau_z+1}^{t_{z+1}-1}) +1. 
\end{align}
\end{proof}

We now complete the proof of Theorem 3 by showing that the conditions that determine the switching time between $\ftl$ and $\wcalg$ are appropriately chosen to upper bound $\ftlsum_{t_z:\tau_z-1}$ and $\Reg(\wcalg, \ell_{\tau_z+1}^{t_{z+1}-1})$.
Firstly, note that for any $z < \zlast$, $\Reg(\wcalg, \ell_{\tau_z+1}^{t_{z+1}-1}) > g(L_z^*)$, which implies that $L^* > L_z^*$. In particular, the epoch $\zlast-1$ was exited, which implies that\footnote{Here we have implicitly assumed that $L^* > \log m$---if not, then there is only one epoch, $\zlast = 0$ and the result is readily implied by Corollary 2.} 
\[
2^{\zlast-1} \log m \le L^* \implies \zlast \le \log\left(\frac{L^*}{\log m}\right) + 1
\]

By the stopping condition of the epoch, $\Reg(\wcalg, \ell_{\tau_z+1}^{t_{z+1}-1}) \leq \Reg(\wcalg, \ell_{\tau_z+1}^{t_{z+1}-2}) + 1 \leq g(L_z^*) + 1$, such that substituting into~\cref{small_loss_regret_lemma} implies
\begin{align}
\Reg(\bob, \ell^n)  \leq \sum_z \left(\ftlsum_{t_z:\tau_z-1} + g(L_z^*) + 2\right).
\end{align}


Also, it always holds that $\ftlsum_{t_z:\tau_z-1} \leq g(L_z^*)$ and for $z < \zlast$, $\ftlsum_{t_z:\tau_z} > g(L_z^*)$. Therefore 
\[\sum_z^{\zlast-1} g(L_z^*) < \sum_z \ftlsum_{t_z:\tau_z-1} \leq \Reg(\ftl, \ell^n)\]
and $\sum_z \ftlsum_{t_z:\tau_z-1} \leq \sum_z^{\zlast} g(L_z^*)$.

To put it all together, if $\sum_z^{\zlast} g(L_z^*) \leq \Reg(\ftl,\ell^n)$, then 
\begin{align}
\Reg(\bob, \ell^n)  \leq 2 \sum_z g(L_z^*) + 2 \zlast.
\end{align}

If $\Reg(\ftl,\ell^n) < \sum_z^{\zlast} g(L_z^*)$, then it must be that in the last epoch the algorithm never switches to $\wcalg$. If it switched to $\wcalg$ it would imply that $\Reg(\ftl,\ell^n) \geq \sum_z \ftlsum_{t_z:\tau_z} > \sum_z^{\zlast} g(L_z^*)$ which would violate the assumption that $\Reg(\ftl,\ell^n) < \sum_z^{\zlast} g(L_z^*)$. Therefore it must be that
\begin{align}
\Reg(\bob, \ell^n)  \leq 2 \Reg(\ftl, \ell^n) + 2 \zlast.
\end{align}

As a result, it follows that (putting the $L^* \le \log m$ and the $L^* > \log m$ cases together) 
\begin{align}
\Reg&(\bob, \ell^n)  \nonumber\\
&\leq 2 \min\left(\Reg(\ftl, \ell^n), \sum_{z = 0}^{\log\left(1+\frac{L^*}{\log m}\right)+1} g(2^z \log m)\right) + 2 \log\left(1+\frac{L^*}{\log m}\right)+2.
\end{align}

\subsection{Proof of Corollary~\ref{cor:SqrtLStarInstantiation}}
 This follows from Theorem~\ref{thm:Regret_small_loss} and by calculating 
    \begin{align}
        \sum_{z=0}^{\log\left(1+\frac{L^*}{\log m}\right)+1} &g(2^z \log m) \nonumber\\
        &= 2\sqrt{2}\log m \sum_{z=0}^{\log\left(1+\frac{L^*}{\log m}\right)+1} 2^{z/2} + \kappa \log m \log\left(1+\frac{L^*}{\log m}\right) + \kappa \log m\nonumber\\
        &\le  \log m \frac{4\sqrt{2}}{\sqrt{2}-1}\sqrt{1+\frac{L^*}{\log m}} + \kappa \log m \log\left(1+\frac{L^*}{\log m}\right)+ \kappa \log m \nonumber \\
        &\le 10 \sqrt{2\log^2 m + 2L^* \log m} + \kappa \log m \log\left(1+\frac{L^*}{\log m}\right) + \kappa \log m \nonumber\\
        &\stackrel{\text{(a)}}{\le} 10 \sqrt{2L^* \log m} + \kappa\log\left(1+\frac{L^*}{\log m}\right)\log m + 10\sqrt{2} \log m + \kappa \log m \nonumber
    \end{align}
    where (a) follows since for nonnegative $a,b$ $\sqrt{a+b} \le \sqrt{a}+ \sqrt{b}$.
\section{Proof of Theorem~\ref{thm:LowerBdCR}} \label{appendix:ConversePf}
    Consider a large even $n$. We then have for horizon size $n+1$, $\Reg(\ftl,y^{n+1}) = \frac{c(y^n)}{2}$.  Moreover, $\Reg(\Cover,y^{n+1}) = f_{n+1}$. Let\footnote{Note that $k$ in this definition does not counting the origin as as a line crossing.}
    \begin{align}\label{eq:PnkDefn}
        p_{n,k} \defeq \Prob[c(\e^{n}) = k+1].
    \end{align}
    We then have,
    \begin{align}
     \E[\min\{c(\e^n), 2f_{n+1}\}] &= \sum_{k=1}^n \min\{k, 2f_{n+1}\} \Pr[c(\e^n) = k]\nonumber\\
        &= \sum_{k=1}^n \min\{k, 2f_{n+1}\} p_{n,k-1} \nonumber\\
        &= \sum_{k=0}^n \min\{k+1,2f_{n+1}\}p_{n,k} \nonumber\\
        &= \sum_{k=0}^{\lfloor 2f_{n+1} \rfloor-1} (k+1)p_{n,k} + 2f_{n+1} \Prob[c(\e^n) \ge 2f_{n+1}] \nonumber\\
        &= \sum_{k=0}^{\lfloor 2f_{n+1} \rfloor-1} (k+1)p_{n,k} + 2f_{n+1} \Prob[c(\e^n) \ge 2f_{n+1}] + \Prob[c(\e^n) \le \lfloor 2f_{n+1}\rfloor] \nonumber\\
        &=  \sum_{k=0}^{\lfloor 2f_{n+1} \rfloor-1} kp_{n,k} + 2f_{n+1} \left(1-\sum_{k=0}^{\lfloor 2f_{n+1} \rfloor-1} p_{n,k}\right) +  \Prob[c(\e^n) \le \lfloor 2f_{n+1}\rfloor] \label{eq:ConvSplit}
    \end{align} 
    Upon dividing~\cref{eq:ConvSplit} by $2f_{n+1}$, we get 
    \begin{align} \label{eq:ConvThreeTerms}
      \frac{\E[\min\{c(\e^n), 2f_{n+1}\}]}{2f_{n+1}}  = \frac{\sum_{k=0}^{\lfloor 2f_{n+1} \rfloor-1} kp_{n,k}}{2f_{n+1}} +  \left(1-\sum_{k=0}^{\lfloor 2f_{n+1} \rfloor-1} p_{n,k}\right) +  \frac{\Prob[c(\e^n) \le \lfloor 2f_{n+1}\rfloor]}{2f_{n+1}}.
    \end{align}
    Note that the third term vanishes since $f_{n+1} \to \infty$ and $0 \le \Prob[\cdot] \le 1$. We will now separately evaluate the first two terms in~\cref{eq:ConvThreeTerms}. To do this, we require an auxiliary lemma, the proof of which is provided later.
    \begin{lemma}\label{lem:pnkasymptotics}
        If $k \le C\sqrt{n}$ for an absolute constant $C$, then for large enough $n$ ($n \ge 32C^2$ suffices) we have
        \begin{align}\label{eq:pnkDensityApprox}
            e^{-\frac{16C^3}{\sqrt{n}}} \le \frac{p_{n,k}}{\sqrt{\frac{2}{n\pi}}e^{-k^2/2n}} \le \sqrt{\frac{1-C/\sqrt{n}}{1-2C/\sqrt{n}}} e^{\frac{16C^3}{\sqrt{n}}}.
        \end{align}
        That is,
        \[
        p_{n,k} = \sqrt{\frac{2}{n\pi}}e^{-k^2/2n} (1+o(1)).
        \]
    \end{lemma}
    We now evaluate the first term in~\cref{eq:ConvThreeTerms}. Since $\lfloor 2f_{n+1} \rfloor - 1 \le 2\sqrt{n}$, we invoke~\cref{lem:pnkasymptotics} to evaluate 
    \begin{align}
        \sum_{k=0}^{\lfloor 2f_{n+1} \rfloor-1} kp_{n,k} &=    (1+o(1))\sum_{k=0}^{\lfloor 2f_{n+1} \rfloor-1} k \sqrt{\frac{2}{n\pi}} e^{-\frac{k^2}{2n}} \\
        &= (1+o(1)) \sqrt{\frac{2}{n\pi}} \sum_{k=0}^{\lfloor 2f_{n+1} \rfloor-1} k e^{-\frac{k^2}{2n}} \label{eq:PartialSum}
    \end{align}
    Now, we note that since $x \mapsto xe^{-\frac{x^2}{2n}}$ is increasing on $(0,\lfloor 2f_{n+1} \rfloor-1)$ 
    , by a Riemann approximation, we have that 
    \begin{align} \label{eq:RiemannApprox}
    \int_{0}^{2f_{n+1}-2} x e^{-\frac{x^2}{2n}} dx - 1 \le\sum_{k=0}^{\lfloor 2f_{n+1} \rfloor-1} k e^{-\frac{k^2}{2n}} \le \int_{0}^{2f_{n+1}} x e^{-\frac{x^2}{2n}} dx 
    \end{align}
    and evaluating
    \begin{align}
        \int_{0}^{2f_{n+1}} x e^{-\frac{x^2}{2n}} dx &= \frac{1}{2} \int_{0}^{4f_{n+1}^2} e^{-\frac{t}{2n}} dt \nonumber\\
        &= n\left(1-e^{-\frac{4f_{n+1}^2}{2n}}\right) \nonumber\\
        &= n\left(1-e^{-\frac{1}{\pi}(1+o(1))}\right).
    \end{align}
    Evaluating the lower bound in~\cref{eq:RiemannApprox} analogously we have 
    \begin{align}\label{eq:ExptZCLowRegime}
        \sum_{k=0}^{\lfloor 2f_{n+1} \rfloor-1} k e^{-\frac{k^2}{2n}} = (1+o(1)) n\left(1-e^{-\frac{1}{\pi}(1+o(1))}\right)
    \end{align}
    and therefore from~\cref{eq:PartialSum}, 
    \begin{align}
        \frac{\sum_{k=0}^{\lfloor 2f_{n+1} \rfloor-1} k p_{n,k}}{2f_{n+1}} &= (1+o(1))\frac{\left(1-e^{-\frac{1}{\pi}(1+o(1))}\right) n \sqrt{\frac{2}{n \pi}}}{\sqrt{\frac{2(n+1)}{\pi}}} \nonumber\\
      &= (1+o(1))\left(1-e^{-\frac{1}{\pi}(1+o(1))}\right) \label{eq:LittleOPartialSum}
    \end{align}
    and therefore, from~\cref{eq:LittleOPartialSum}, we have that 
    \begin{align}\label{eq:LimFirstTermConverse}
        \lim_{n \to \infty}  \frac{\sum_{k=0}^{\lfloor 2f_{n+1} \rfloor-1} k p_{n,k}}{2f_{n+1}}  = 1-e^{-\frac{1}{\pi}}.
    \end{align}
    We now address the second term in~\cref{eq:ConvThreeTerms} by invoking~\cref{lem:pnkasymptotics} and noting that 
    \begin{align}
       \sum_{k=0}^{\lfloor 2f_{n+1} \rfloor-1} p_{n,k} &= (1+o(1)) \sqrt{\frac{2}{n \pi}}\sum_{k=0}^{\lfloor 2f_{n+1} \rfloor-1} e^{-\frac{k^2}{2n}} \nonumber\\
       &\stackrel{(a)}{=} (1+o(1)) \int_{0}^{2f_{n+1}} e^{-\frac{x^2}{2n}} dx \nonumber\\
       &= (1+o(1)) \sqrt{\frac{2}{\pi}} \int_{0}^{\sqrt{\frac{2}{\pi}}} e^{-\frac{t^2}{2}} dt \nonumber\\
       &= \frac{1}{\sqrt{2\pi}} \int_{-\sqrt{\frac{2}{\pi}}}^{\sqrt{\frac{2}{\pi}}} e^{-\frac{t^2}{2}} dt \nonumber \\
        &= \Prob\left(-\sqrt{\frac{2}{\pi}} \le X \le \sqrt{\frac{2}{\pi}}\right) \label{eq:StandardNormalConfInt}
    \end{align}
    where in~\cref{eq:StandardNormalConfInt} $X \sim \mathcal{N}(0,1)$. Then, 
    \begin{align}
        1- \sum_{k=0}^{\lfloor 2f_{n+1} \rfloor-1} p_{n,k}  &\to 1-\Prob\left(-\sqrt{\frac{2}{\pi}} \le X \le \sqrt{\frac{2}{\pi}}\right) \nonumber\\
        &= 2Q\left(\sqrt{\frac{2}{\pi}}\right) \label{eq:LimSecTermConverse}. 
    \end{align}
    Substituting~\cref{eq:LimFirstTermConverse} and~\cref{eq:LimSecTermConverse} in~\cref{eq:ConvThreeTerms} yields that 
    \begin{align}
     \frac{1}{\g_n} = \frac{\E[\min\{c(\e^n), 2f_{n+1}\}]}{2f_{n+1}} \to 1-e^{-\frac{1}{\pi}} + 2Q\left(\sqrt{\frac{2}{\pi}}\right)
    \end{align}
    and therefore 
    \begin{align}\label{eq:finalGammaLimit}
        \g_n \to \frac{1}{1-e^{-\frac{1}{\pi}} + 2Q\left(\sqrt{\frac{2}{\pi}}\right)}. 
    \end{align}

\subsection{Proof of Lemma \ref{lem:pnkasymptotics}}
    We first note the following. 
    \begin{proposition}[\cite{Feller1991}, Chapter 3, Exercise 11]\label{prop:pnkexact}
        \[
        p_{n,k} = \frac{1}{2^{n-k}} \binom{n-k}{n/2}.
        \]
    \end{proposition}
    Therefore, 
    \[
    \frac{p_{n,k}2^n}{2^k} = \binom{n-k}{n/2} = \frac{(n-k)!}{\frac{n}{2}!\left(\frac{n}{2}-k\right)!}.
    \]
    We now use the Stirling approximation:
    \begin{align}\label{eq:stirling}
        \sqrt{2\pi m} \left(\frac{m}{e}\right)^m e^{\frac{1}{12m+1}} \le m! \le     \sqrt{2\pi m} \left(\frac{m}{e}\right)^m e^{\frac{1}{12m}}.
    \end{align}
    Using~\cref{eq:stirling} we have 
    \begin{align}
         \frac{p_{n,k}2^n}{2^k} &\le \frac{\sqrt{2\pi(n-k)}}{\sqrt{2\pi(n/2)} \sqrt{2\pi(n/2-k)}} \cdot \frac{(n-k)^{n-k}}{(n/2)^{n/2} (n/2-k)^{n/2-k}} \nonumber\\
         &\qquad\qquad\qquad\qquad\qquad\qquad\cdot \exp{\left(\frac{1}{12(n-k)}-\frac{1}{6n+1}-\frac{1}{6n-12k+1}\right)} \nonumber\\
         &= \sqrt{\frac{2}{n\pi}} \sqrt{\frac{n-k}{n-2k}} \frac{(n-k)^{n-k}2^{n-k}}{n^{n/2}(n-2k)^{n/2-k}}\cdot\exp{\left(\frac{1}{12n-k}\right)} \nonumber\\
         &\stackrel{(a)}{\le} \sqrt{\frac{2}{n\pi}} \cdot 2^{n-k} \cdot\frac{(n-k)^{n-k}}{n^{n/2}(n-2k)^{n/2-k}} \exp\left(\frac{1}{12n-k}\right) \cdot\sqrt{\frac{1-C/\sqrt{n}}{1-2C/\sqrt{n}}} \nonumber\\
         &=  \sqrt{\frac{2}{n\pi}} \cdot 2^{n-k} \underbrace{\frac{(1-k/n)^{n-k}}{(1-2k/n)^{n/2-k}}}_{=: T}\exp\left(\frac{1}{12n-k}\right)\cdot\sqrt{\frac{1-C/\sqrt{n}}{1-2C/\sqrt{n}}}.  \label{eq:SubstituteT}
    \end{align}
    We now analyze the term $T$ in~\cref{eq:SubstituteT} in more detail. 
    \begin{proposition}\label{prop:TaylorSeriesT}
    \begin{align} \label{eq:LnTBds}
            -\frac{k^2}{2n^2} - c\frac{k^3}{n^2} \le \ln T \le  -\frac{k^2}{2n^2} + c\frac{k^3}{n^2}
    \end{align}
    where $c \le 15$. 
    \end{proposition}
    \begin{proof}
        By Taylor theorem, we have 
        \begin{align}\label{eq:lnXBds}
            \ln(1-x) = -x - \frac{x^2}{2} - \frac{x^3}{3(1-\mu)^2} \text{ for } \mu \in (0,x).
        \end{align}
        Therefore, for $k \le C\sqrt{n}$ and $n \ge 16C^2$ we have 
        \begin{align}
            \ln\left(1-\frac{k}{n}\right) &= -\frac{k}{n} - \frac{k^2}{2n^2} - \frac{\alpha_1k^3}{n^3} \label{eq:OneTermT}\\
            \ln\left(1-\frac{2k}{n}\right) &= -\frac{2k}{n} - \frac{2k^2}{2n^2} - \frac{\alpha_2k^3}{n^3} \label{eq:SecondTermT}
        \end{align}
        for $\alpha_1,\alpha_2 \in \Big(\frac{1}{3},\frac{8}{3}\Big]$. Evaluating $\ln T$ we have 
        \begin{align}
            \ln T &= (n-k)\ln\left(1-\frac{k}{n}\right) - \left(n-\frac{k}{2}\right)\ln\left(1-\frac{2k}{n}\right)  \nonumber\\
            &\stackrel{(a)}{=} (n-k)\left(-\frac{k}{n} - \frac{k^2}{2n^2} - \frac{\alpha_1k^3}{n^3}\right) - \left(n-\frac{k}{2}\right)\left(-\frac{2k}{n} - \frac{2k^2}{2n^2} - \frac{\alpha_2k^3}{n^3}\right) \nonumber\\
            &= \left(-\frac{k^2}{2n}+\frac{k^2}{n}+\frac{k^2}{n}-\frac{2k^2}{n}\right)+\left(-\alpha_1+\frac{1}{2}+\frac{\alpha_2}{2}-2\right)\frac{k^3}{n^2}+\left(\alpha_1-\alpha_2\right)\frac{k^4}{n^3} \nonumber\\
            &= -\frac{k^2}{2n^2} + c_1\frac{k^3}{n^2} + c_2\frac{k^4}{n^3} \label{eq:subc1c2}
        \end{align}
        where $(a)$ follows by substituting~\cref{eq:OneTermT} and~\cref{eq:SecondTermT}. Now, since $\frac{k^4}{n^3} \le \frac{k^3}{n^2}$ we have 
        \begin{align}
         \ln T  \le  -\frac{k^2}{2n^2} +  (|c_1|+|c_2|)\frac{k^3}{n^2}
        \end{align}
        and on the other hand for the same reason 
        \begin{align}
            \ln T \ge -\frac{k^2}{2n} - (|c_1|+|c_2|)\frac{k^3}{n^2}
        \end{align}
        The proposition follows by noticing $|c_1|+|c_2| \le 15$ by using that $\alpha_1, \alpha_2 \in \left(\frac{1}{3},\frac{8}{3}\right]$.
    \end{proof}
Using~\cref{prop:TaylorSeriesT} in~\cref{eq:SubstituteT} we have the upper bound 
\begin{align}
    p_{n,k} &\le \sqrt{\frac{2}{n\pi}} \cdot \exp\left(-\frac{k^2}{2n}+\frac{15k^3}{n^2}\right)\cdot \exp\left(\frac{1}{12n-k}\right)\cdot\sqrt{\frac{1-C/\sqrt{n}}{1-2C/\sqrt{n}}} \nonumber\\
    &\stackrel{(a)}{\le} \sqrt{\frac{2}{n\pi}} \cdot \exp\left(-\frac{k^2}{2n}\right)\cdot \exp\left(\frac{16k^3}{n^2}\right)\cdot\sqrt{\frac{1-C/\sqrt{n}}{1-2C/\sqrt{n}}} \nonumber\\
    &\stackrel{(b)}{\le}  \sqrt{\frac{2}{n\pi}} \cdot \exp\left(-\frac{k^2}{2n}\right)\cdot \exp\left(\frac{16C^3}{\sqrt{n}}\right)\cdot\sqrt{\frac{1-C/\sqrt{n}}{1-2C/\sqrt{n}}} \label{eq:pnkAsympUB}
\end{align}
where $(a)$ uses $\frac{1}{12n-k}+\frac{15k^3}{n^2} \le \frac{16k^3}{n^2}$, and $(b)$ uses the fact that $k \le C\sqrt{n}$, which yields the upper bound in~\cref{eq:pnkDensityApprox}. The lower bound follows analogously by the Stirling approximation~\cref{eq:stirling} and~\cref{prop:TaylorSeriesT}.